\documentclass[envcountsame,runningheads]{llncs}
\usepackage[utf8]{inputenc}
\usepackage[T1]{fontenc}

\usepackage{graphicx}
\usepackage[svgnames]{xcolor}
\usepackage{fca}
\usepackage{mathtools}
\usepackage{marvosym}
\usepackage{cite}
\usepackage[hidelinks]{hyperref}
\usepackage{multirow, makecell}

\usepackage{enumitem}
\usepackage{booktabs}
\usepackage{tikz}
\usepackage{tikzscale}
\usetikzlibrary{arrows.meta}
\usepackage[outline]{contour}
\contourlength{2.5pt}
\dimendef\prevdepth=0
\usepackage{svg}
\usepackage[ruled,vlined]{algorithm2e}
\usepackage{cleveref}

\usepackage{todonotes}
\usepackage{centernot}

\usepackage{subcaption}
\usepackage{acro}

\usepackage{tabularx}
\newcolumntype{x}[1]{!{\centering\arraybackslash\vrule width #1}}

\newcommand{\old}[1]{{\color{teal} #1}}

\renewcommand{\old}[1]{ }
\newcommand{\commentout}[1]{ }

\newcommand{\ie}{i.\,e.\xspace}

\newcommand{\eg}{e.\,g.\xspace}
\newcommand{\K}{\mathbb{K}}

\newcommand{\B}{\mathfrak{B}}
\newcommand{\uB}{\underline{\mathfrak{B}}}

\newcommand{\coveredby}{\prec}
\newcommand{\ncoveredby}{\centernot\prec}
\newcommand{\covers}{\succ}

\newcommand{\uP}{\boldsymbol{P}}

\newcommand{\uL}{\boldsymbol{L}}

\let\subset\subseteq
\let\cref\Cref

\newcommand{\up}{\vee}
\newcommand{\down}{\wedge}

\newcommand{\proofbox}{\hspace*{\fill}$\Box$}

\DeclareAcronym{fca}{
  short = FCA,
  long = Formal Concept Analysis
}
\DeclareAcronym{dm}{
  short = DM,
  long = Dedekind-MacNeille completion
}

\newcommand{\ColoredEdge}[3]{{\color{#3}\Edge{#1}{#2}}}
\newcommand{\ColoredleftObjbox}[5]{{\color{#5}\leftObjbox{#1}{#2}{#3}{#4}}}
\newcommand{\ColoredrightAttbox}[5]{{\color{#5}\rightAttbox{#1}{#2}{#3}{#4}}}

\newcommand{\edgerisethickness}{1pt}
\newcommand{\edgenorisecolor}{gray}
\newcommand{\edgebothrisecolor}{gray}
\newcommand{\edgeobjectrisecolor}{gray}
\newcommand{\edgeattributerisecolor}{gray}

\newcommand{\nodeattributerisecolor}{red}
\newcommand{\nodeobjectrisecolor}{red}

\newcommand{\meetnumberoffset}{-10}
\newcommand{\joinnumberoffset}{-10}

\tikzset{
  pics/stepnumber/.style args={#1/#2/#3/#4}{
      code={
          \draw[black] (0,0) circle (#3);
          \draw[-] (-#3,0) -- (#3,0);
          \node at (0,0.5*#3) {#1};
          \node at (0,-0.5*#3) {#2};
          \node[draw=none, minimum size=1.4*#3] (#4) at (0,0) {};
        }
    }
}

\tikzset{
  pics/stepnumber-color/.style args={#1/#2/#3/#4/#5/#6}{
      code={
          \draw[black] (0,0) circle (#3);
          \fill[#5] (0,0) -- (#3,0) arc[start angle=0, end angle=180, radius=#3] -- cycle;
          \fill[#6] (0,0) -- (#3,0) arc[start angle=0, end angle=-180, radius=#3] -- cycle;
          \draw[-] (-#3,0) -- (#3,0);
          \node at (0,0.5*#3) {#1};
          \node at (0,-0.5*#3) {#2};
          \node[draw=none, minimum size=1.4*#3] (#4) at (0,0) {};
        }
    }
}

\begin{document}

\title{Rises for Measuring Local Distributivity\\ in Lattices}

\author{Mohammad Abdulla\inst{1,2}, Tobias Hille\inst{1,2},\\ Dominik
  Dürrschnabel\inst{1,2}, Gerd Stumme\inst{1,2}}
\institute{Knowledge \& Data Engineering Group, University of Kassel, Germany
  \and Interdisciplinary Research Center for Information Systems Design,
  \mbox{University of Kassel, Germany} \\%
  \email{\{abdulla, hille, duerrschnabel, stumme\}@cs.uni-kassel.de} }
\maketitle

\begin{abstract}

  Distributivity is a well-established and extensively studied notion in lattice theory. In the context of data analysis, particularly within Formal Concept Analysis (FCA), lattices are often observed to exhibit a high degree of distributivity. However, no standardized measure exists to quantify this property. In this paper, we introduce the notion of rises in (concept) lattices as a means to assess distributivity. Rises capture how the number of attributes or objects in covering concepts change within the concept lattice. We show that a lattice is distributive if and only if no non-unit rises occur.
  Furthermore, we relate rises to the classical notion of meet- and join distributivity. We observe that concept lattices from real-world data are to a high degree join-distributive, but much less meet-distributive. We additionally study how join-distributivity manifests on the level of ordered sets.


\end{abstract}

\section{Introduction}
\label{sec:introduction}

In lattice theory, a particular interesting property of a lattice $(L\leq)$ is that of being distributive, \ie,  when, for all
$x,y,z \in L$, $x\down (y\up z)=(x\down y)\up(x\down z)$ and $x\up(y\down z)=(x\up y)\down (x\up z)$ hold.
Distributive lattices are equipped with strong properties that make them particularly useful in data analysis and knowledge representation: They can for instance be embedded into products of total orders, which can be used for a divide-and-conquer approach for data analysis by decompositions or for well readable visualisations; moreover for their implicational theory, only implications with one-element premises --- which are thus easy to understand --- have to be considered.

Previous studies suggest that lattices arising in data-driven contexts, such as
those in Formal Concept Analysis (FCA), exhibit a high degree of distributivity.
Wille~\cite{WilleTruncated2003} provides evidence that many concept lattices
from real-world data sources are `mostly distributive'.
Despite this observation, there is no widely accepted measure to quantify the
degree of distributivity in such lattices --- not to speak of any method for `repairing' the non-distributive part of the lattice.

For data analysis tasks such as decompositions and visualisations, it might also suffice to consider some weaker form of distributivity, such as semidistributivity, local distributivity, modularity, semimodularity, etc. It should be obvious that `the right way' of measuring and repairing non-distributivity (and any other algebraic property) will strongly depend on the goal of the data analysis. Therefore a large tool kit with different measures and methods is required.

In~\cite{abdulla2024birkhoff}, we started this approach by proposing the Birkhoff completion, which embeds any (non-distributive) lattice into the smallest possible distributive one by preserving all joins. In this paper, we follow another observation: (structural properties of) neighboring elements in distributive (concept) lattices vary only slightly.
When we move in a finite distributive lattice from an element to an upper neighbor, the set of join-irreducible elements below increases exactly by one [and the set of meet-irreducible elements above decreases exactly by one]. (In FCA that means that the cardinalities of the extents and intents of two neighboring concepts of a reduced context differ exactly by one.) Hence whenever a rise (which we will formally define in Section \ref{sec-local}) is larger than one, then the distributivity law is violated --- which could be repaired in a subsequent step.

It turned out that we rediscovered with the property `all join rises are equal to one' an old concept which has already frequently been rediscovered before \cite{monjardet1985use} --- the so-called \emph{local distributivity} or \emph{join-distributivity}.  Join-distributivity has initially been introduced by Dilworth in 1940 (see \cite{Dilworth1990}) and has since then been discovered and studied independently from many different perspectives. In particular,  join-distributive lattices are equivalent to anti-matroids, convex geometries \cite{edelman1980meet,edelman1985theory,korte1991abstract} and learning spaces \cite{falmagne2010learning,wild2017compressed}. The equivalences summarized in Theorem~\ref{theorem-joindistributivity} in Section~\ref{sec-local} are both reason and result of this confluence of perspectives.

In this paper, we will study specific aspects of join- and meet-rises and join-/meetdistributivity that may be beneficial for an analysis of (concept) lattices that are `almost join-distributive'. Because of the duality principle of lattice theory,\footnote{Every valid statement also holds in its dual version, \ie when we systematically flip $\leq \,\,\leftrightarrow\,\, \geq$, `meet' $(\wedge)$ $\leftrightarrow$ `join' $(\vee)$, $\top \leftrightarrow \bot$, etc.} it is sufficient to describe our results only for either join-distributivity or meet-distributivity. Because real-world concept lattices are closer to being join-distributive than meet-distributive (see Section~\ref{sec-realworld} for details) and because learning spaces are also joindistributive, we will focus on the former ones. We will use the term `local distributivity' when talking about the overall picture, and the equivalent term `joindistributivity' for details.

As our research is driven by a data analysis perspective, \emph{all sets and structures in this paper will be finite}, unless explicitly mentioned otherwise.

While the work in this paper is mostly of theoretical nature, we will discuss potential applications in the outlook (Sect.~\ref{sec-conclusion}).
In Sect.~\ref{sec-basics}, we will introduce the necessary background on ordered sets, lattices, and Formal Concept Analysis.
In Sect.~3, we will define rises and establish their connection to local distributivity. In Section 4, we will extend the study of join-distributivity to posets. Finally, Sect.~5 will analyze real-world datasets to quantify local distributivity in practice.
Related work is discussed throughout the paper.

\section{Preliminaries}
\label{sec-basics}
Here we briefly recall the basic notations.
For more background on lattice theory and Formal Concept Analysis, we refer to \cite{Birkhoff,Lattices,roman2008lattices,GanterFormal2024}. We follow the notation of \cite{GanterFormal2024} whenever possible.

\subsection{Ordered Sets and Lattices}
An \emph{order} (sometimes also called \emph{partial order}) on a set $P$ is a binary
relation $\leq$ on $P$ that is: reflexive (For all \(a \in P\), \(a \leq a\)),
antisymmetric (If \(a \leq b\) and \(b \leq a\), then \(a = b\)), transitive (If
\(a \leq b\) and \(b \leq c\), then \(a \leq c\)). $\uP \coloneqq (P,{\leq})$ is
then called an \emph{ordered set}.
We call $(x, y)\in P\times P$ a \emph{comparable pair} of P if $x<y$. We call a comparable pair a \emph{covering pair} and write \( x \coveredby y \) if, for all \( p \in P
\), \( x \leq p < y \) implies \( x = p \).
A \emph{chain} in $(P,{\leq})$ is a set $C\subset P$  where, for all $x,y\in C$, $x\leq y$ or $y\leq x$. %
A chain is called \emph{maximal} if it is not a proper subset of any other chain. %
The \emph{height} $h$ of a finite ordered set $P$ is the length of its longest chain.
The \emph{height of an element} $x$ is given by $\max\{|C| \mid \forall y \in C: y < x$, C is a chain$\}$.
For two elements $x, y\in P$,  $[x, y]:=\{z\mid x\leq z\leq y \}$ is the
\emph{interval} between $x$ and $y$.

If, for $x,y\in P$ , the set $\{z \mid z \leq x\textrm{ and } z \leq y\}$ has a largest element, then it is the \emph{meet} of $x$ and $y$, denoted by $x\wedge y$.
Dually, if there is a least element in $\{z\in P \mid z\geq x, z\geq y\}$, then it is the
\emph{join} of $x$ and $y$, denoted by $ x\vee y$.
An ordered set
$\uL\coloneq(L,{\leq})$ is called a \emph{lattice} if for any two elements their
meet and join exist.
It is a \emph{complete lattice} if all
subsets $X\subseteq L$ have meet (denoted by $\bigwedge X$) and join
(denoted by $\bigvee X$).
$\bot:=\bigwedge \emptyset$ is the \emph{bottom element} and  $\top\coloneqq \bigvee \emptyset$ is the \emph{top element}. $\bot$ and $\top$ are often denoted by 0 and 1. We refrain from this notation, as it may lead to confusion with our concept of rises.

$x\in\uP$ is join-irreducible if it is not the join of a set of elements not containing $x$.
It is meet-irreducible if it is not the meet of a set of elements not containing it.
The set of all join-irreducible elements of $\uP$ is denoted by
$\mathcal{J}(\uP)$, and the set of all its meet-irreducibles is denoted by
$\mathcal{M}(\uP)$.
We define $\mathcal{M}^x
    \coloneqq \{m \in \mathcal{M}(\uP) \mid x \leq m\}.$ %
For $x \in P$ let $m(x) = |\mathcal{M}^x|$. %
Dually, let $\mathcal{J}_x=\{j\in\mathcal{J}(\uP)\mid j\leq x\}$ and for $x\in
    P$ let $ j(x)=|\mathcal{J}_x|$. %
$X\subseteq L$ is an \emph{irredundant} $\bigvee$-representation of $a\in L$ if $a=\bigvee X$ and
$a\neq \bigvee (X\setminus{\{x\}})$ for every $x \in X$.
A \emph{rank function} in a lattice $(L,{\leq})$ is a mapping $r\colon L \to \mathbb{N}$ such that $r(\bot)=0$ and, for any two
elements $x,y \in L$,  $x \coveredby y \Rightarrow r(x) + 1
    =r(y)$ holds. %
A lattice with rank function is \emph{graded}.
$K\subseteq L$ is a \emph{sublattice} of lattice $(L,\leq)$ if $K$ is closed under meets and joins.
Distributivity, as defined in the introduction, can be characterized by forbidden
sublattices, as famously discovered by Birkhoff:
\begin{theorem}[Birkhoff~\!\!\cite{Birkhoff}]
    \label{theorem-M3N5}
    A lattice $\uL$ is distributive if and only if it has neither $M_3$ nor $N_5$ as a
    sublattice.
\end{theorem}

A lattice is called \emph{modular} if for all $x, y, z\in L$
with $z\leq x$ it follows that $x\wedge (y\vee z)=(x\wedge y)\vee z$.
It is well-known that distributive lattices are modular. %
Modularity can be weakened to \emph{semi-modularity} as follows. A finite
lattice $\uL$ is \emph{semimodular} if, for all $x,y \in L$, $x \wedge y
    \coveredby y$ implies $ x \coveredby x \vee y$, and it is \emph{dually semimodular} if
$y \coveredby x \vee y$ implies $ x \wedge y \coveredby x$. The
lattice \( S_7 \) (see~\cref{fig:2x2grid} (c)) is the smallest semimodular lattice which is not modular. %

\begin{theorem}[\!\!\cite{roman2008lattices}]
    A finite lattice $\uL$ is modular if and only if it is both semimodular and
    dually semimodular.
\end{theorem}

A complete lattice $\uL$ is called \emph{join-distributive} if it is semimodular and
every modular sublattice is distributive (see Definition~\ref{def-joindistributive}). Lattices satisfying the dual condition
are called \emph{meet-distributive}.
For these two properties, there are multiple other possible definitions, as listed below in \cref{theorem-joindistributivity}.

\subsection{Formal Contexts and Concept Lattices}
In this section, we introduce the fundamental definitions and concepts from
Formal Concept Analysis (FCA)  used throughout
this work.
A \emph{(formal) context} is a triple \( \mathbb{K} = (G, M, I) \), where \( G \) and $M$ are sets whose elements are called \emph{objects} and \emph{attributes}, resp., and \( I
\subseteq G \times M \) is called \emph{incidence relation}.
For \( A \subseteq G \), we define $A' \coloneqq \{
    m \in M \mid \forall g \in A~(g, m) \in I \}$. %
Dually, for $ B \subseteq M$, we define $B' \coloneqq \{ g \in G \mid \forall m \in
    B~(g, m) \in I \}$. %
A \emph{(formal) concept} of \( \mathbb{K} \) is a pair $(A, B)$ such that
$A\subseteq G$, $B\subseteq M$, $A' = B$, and $B' = A$. \( A \) is the \emph{extent} and \( B \) the \emph{intent} of the concept. %
The set of all concepts of \( \mathbb{K} \), ordered by $ (A_1, B_1) \leq
    (A_2, B_2) :\!\iff A_1 \subseteq A_2 $, forms a complete lattice, the
\emph{concept lattice} of \( \mathbb{K} \), denoted by \( \uB(\mathbb{K}) \). A
context  is \emph{clarified} if, for \( g, h \in G \),  \( g' = h' \) implies \( g = h \), and if for  \( m, n \in M \), \( m' = n' \) implies \( m
= n \).
For  $g\in G$,  $(g'',g')$ is an \emph{object concept} and, for $m\in M$, $(m',m'')$ is an \emph{attribute concept}.
A context is  \emph{object-reduced} if every object concept is join-irreducible, it is \emph{attribute-reduced} if every attribute concept is meet-irreducible, and it is \emph{reduced} if it is both. %
For $g \in G$ and $m \in M$ with $(g,m)\notin I$, we write $g \DownArrow m$ if, for $h\in G$, $g' \subsetneq h'$ implies $(h,m) \in I$, and
$g \UpArrow m$ if, for $n\in M$, $m' \subsetneq n'$ implies $(g,n)\in I$. $g \DoubleArrow m$ stands for both $g \DownArrow m$ and $g \UpArrow m$.
The \emph{\acl{dm}} of an ordered set $\uP$, denoted by $DM(\uP)$, is the smallest complete lattice into which $\uP$
can be order embedded: %

\begin{theorem}[FCA version of Dedekind’s Completion Theorem \cite{GanterFormal2024}]
    \label{theorem:dct}
    For an ordered set $\uP$, define $\iota(x) \coloneqq (\{a\in P \mid a\leq x\}, \{a\in P \mid a\geq x\})$ for $x \in
        P$. %
    This defines an embedding $\iota$ of $\uP$ in $\uB(P,P,\leq)$. %
    Moreover, $\iota(\bigvee X) = \bigvee\iota(X)$ or $\iota(\bigwedge X) =
        \bigwedge \iota(X)$ if the join or meet of $X$, respectively, exists in
    $\uP$. Furthermore, if \( \kappa \) is any order embedding of $\uP$ into a
    complete lattice $\uL$, then there exists an embedding \( \lambda \) of \(
    \uB(P, P, \leq) \) into $\uL$ such that $\kappa = \lambda \circ \iota.$

\end{theorem}

\section{Rises and Local Distributivity}
\label{sec-local}

Recall, that the mappings $j$ and $m$ count the number of join-irreducibles and meet-irreducibles below and above an element in a poset, respectively.

\begin{definition}
  Let $\uL:=(L,\leq)$ be a finite lattice.
  The \emph{join-rise} of a comparable pair $(x, y)$ of $\uL$ is denoted by
  $\Delta_j(x,y) \coloneqq j(y)-j(x)$ and the \emph{meet-rise} by
  $\Delta_m(x,y) \coloneqq m(x)-m(y)$.
  We call a rise between covering elements \emph{unit} if its value
  is one, otherwise we call it \emph{non-unit}.
\end{definition}

We will measure these rises mostly for covering pairs  $x\prec y$.

If we already have a clarified and reduced context $\K$, the maps $j$ and $m$ on $\B(K)$ measure the size of the extent and the intent respectively, i.e., $j(A,B)=|\mathcal{J}_{(A,B)}|=|A|$ and $m(A,B)=|\mathcal{M}_{(A,B)}|=|B|$. This observation gives rise to the following lemma.

\begin{lemma}\label{lem:abuse}
  For two formal concepts $(A,B) \leq (C,D)$ of a reduced and clarified context
  $\Delta_j((A,B),(C,D)) = |C| - |A|$ and $\Delta_m((A,B),(C,D)) = |B| - |D|$ hold.
\end{lemma}

Rises capture how concept extents or intents change when we move up
or down in the lattice. For some initial examples, we refer the reader to
\cref{fig:2x2grid}, which shows six example contexts and the corresponding
concept lattices. %
The concepts are labels with the cardinalities of their extent (below) and of their intent (above). Non-unit rises between covering concepts are highlighted. We see that in the distributive lattice $C_2$ all rises are unit, whereas the other three lattices --- which are all non-distributive --- all have at least one non-unit rise. Furthermore, $S_7$ has some non-unit join-rises, while all its meet-rises are unit. (We will see later in  Theorem~\ref{theorem:s7-sublattice} that the existence of $S_7$ of sublattice is also a necessary condition for this situation.)

\begin{figure}[t]
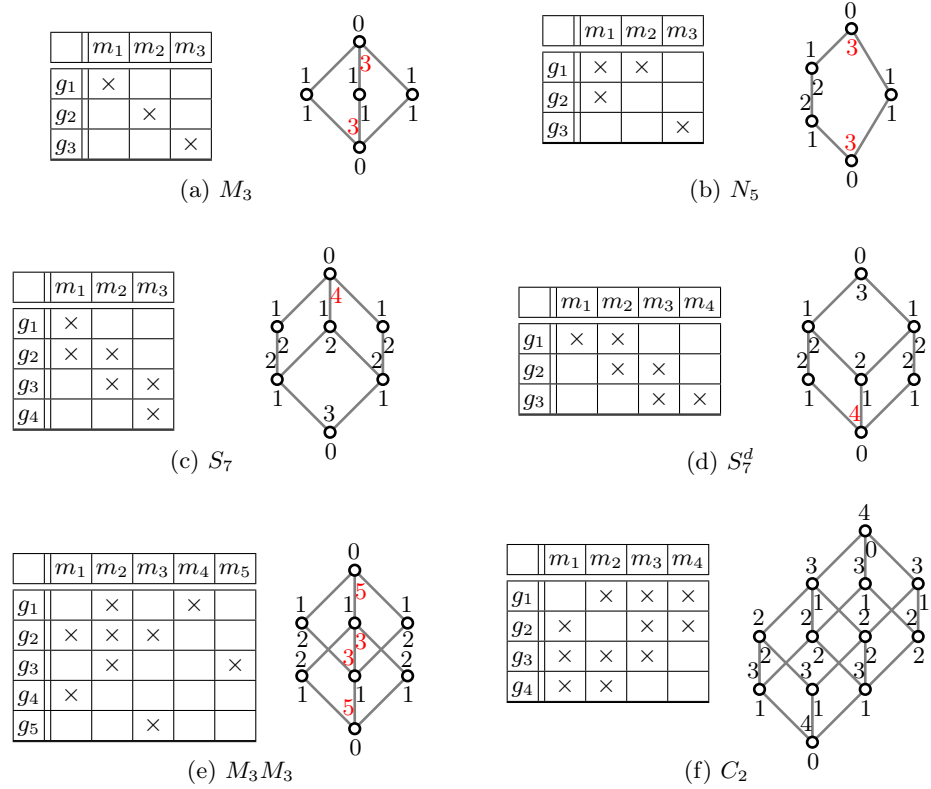

  \centering%
  \subfloat[$M_3$]{%
    \hspace{.5cm}%
    \input{ctx/m3}%
    \hspace{0.8cm}%
    \raisebox{-0.6cm}{\input{figs/m3-rise}}%
    \label{fig-m3}
  }%
  \hfill %
  \subfloat[$N_5$]{%
    \input{ctx/n5}%
    \hspace{.8cm}%
    \raisebox{-1cm}{\input{figs/n5-rise}}%
    \label{fig-n5}
    \hspace{0.5cm}%
  }%

  \bigskip
  \bigskip
  \subfloat[$S_7$]{%
    \input{ctx/s7d}
    \hspace{.8cm}%
    \raisebox{-1.0cm}{\input{figs/s7-rise}}%
    \label{fig-ms7}
  }%
  \hfill %
  \subfloat[$S_7^d$]{%
    \input{ctx/s7}%
    \hspace{0.7cm}%
    \raisebox{-1cm}{\input{figs/s7-dual-rise}}%
    \label{fig:s7-dual-rise}
  }%

  \bigskip
  \bigskip
  \centering%
  \subfloat[$M_3M_3$]{%
    \input{ctx/m3m3}%
    \hspace{1.2cm}%
    \raisebox{-1cm}{\input{figs/m3-m3-rise}}%
    \label{fig:m3-m3-rise}
  }%
  \hfill %
  \subfloat[$C_2$]{%
    \input{ctx/c2}%
    \hspace{0.1cm}%
    \label{fig-c2}
    \raisebox{-1.5cm}{\input{figs/c2-rise}}%
  }%
  \caption{%
    Six small reduced and clarified contexts and their concept lattices. %
    The lower label (object) of a concept node is the number of join
    irreducible elements in the down-set, and correspondingly the upper label
    (attribute) is the number of meet irreducible elements in the up-set. %
    Non-unit rises (see \Cref{lem:abuse}) between covering concepts are
    indicated by a colored label. %
  }

  \label{fig:2x2grid}%
\end{figure}

When studying the behavior of non-unit rises, we quickly recognized that the absence of non-unit rises is just another way of looking at a  well-known (and, as said before, frequently re-discovered) concept:

\begin{definition}
  \label{def-joindistributive}
  A finite lattice $\uL$ is \emph{locally distributive} or \emph{join-distributive} (and sometimes also called \emph{upper locally distributive (ULD)} in the literature) if $\uL$ is semimodular and every modular sublattice is distributive. Lattices satisfying the dual condition are called \emph{meet-distributive}.
\end{definition}

The following theorem provides different views on this --- somewhat complexly defined --- concept. It summarizes many well-known results that are stated (sometimes in the dual version), \eg, in \cite{Dilworth1990,monjardet1985use,edelman1985theory,edelman1980meet,korte1991abstract,SternFrom1999,muehle2023meet,GanterFormal2024}. For sake of simplicity and uniformity, we consider finite lattices only.\footnote{For some of the equivalences in the following theorem, the finiteness condition can be relaxed, \eg, to the finite chain condition or to doubly founded lattices.} As always in lattice theory, the dual theorem also holds.

\begin{theorem}
  \label{theorem-joindistributivity}
  Let $\uL$ be a finite lattice. Then the following are equivalent:
  \begin{enumerate}
    \item $\uL$ is join-distributive.
    \item $\uL$ has a neighborhood-preserving $\bigvee$-embedding into a powerset lattice.
    \item Every element has a unique irredundant $\bigwedge$-representation.
    \item For all $x,y\in L$ with $x\prec y$ there exists a unique meet-irreducible $z$ such that $x\leq z$ and $y\not\leq z$.
    \item $\uL$ is graded with the rank function $r(x)=|\{m\in\mathcal{M}(\uL) \mid m\not\geq x\}|$.
    \item For every element $x\in L$, the interval $[x, x^\top]$ with
          $x^\top:=\bigvee\{y\in L \mid x\prec y\}$ is Boolean.
  \end{enumerate}
  If $\uL = \uB(G,M,I)$, then the following conditions are equivalent to those above:
  \begin{enumerate}
    \setcounter{enumi}{6}
    \item For every concept intent $B$ and for all $m,n\in M\setminus B$ with $m'\neq n'$ we have the \emph{anti-exchange axiom}: $m\in (B\cup\{n\})''$ implies $n\notin (B\cup\{m\})''$.
    \item Every intent $B$ is the closure of its extremal points, where $m\in M$ is an \emph{extremal point} of $B$ if $m\in B$ but $m\notin(B\setminus\{n\in M \mid m' = n'\})''$.
    \item If $m$ and $n$ are irreducible attributes then $g\swarrow m$ and $g\swarrow n$ imply $m'=n'$.
  \end{enumerate}
\end{theorem}

We can now connect our join-rises to this theorem.
\begin{theorem}
  \label{theorem-our-joindistributivity}
  Let $\uL$ be a finite lattice. Then the following condition is equivalent to those stated in Theorem~\ref{theorem-joindistributivity}:
  \begin{enumerate}
    \setcounter{enumi}{9}
    \item For all $x,y\in L$ with $x\prec y$ it holds $\Delta_m(x,y)=1$.
  \end{enumerate}
  If $\uL = \uB(G,M,I)$ and $(G,M,I)$ is clarified and reduced, then the following condition is equivalent to those above:
  \begin{enumerate}
    \setcounter{enumi}{10}
    \item For all concepts $(A,B),(C,D)\in\B(G,M,I)$ with $(A,B)\prec (C,D)$ it holds $|B|-|D|=1$
  \end{enumerate}
\end{theorem}
\begin{proof}
  (4)\,$\Leftrightarrow$\,(10): For all $x,y\in \uL$ with $x\prec y$, the existence of $z$ in Item 4 of Theorem~\ref{theorem-joindistributivity} is equivalent to $\Delta_m(x,y)=1$.

  (10)\,$\Leftrightarrow$\,(11) follows from Lemma~\ref{lem:abuse}. \proofbox
\end{proof}

As it is well-known that a finite lattice is distributive if and only if it is both join-distributive and meetdistributive, we obtain the following necessary and sufficient criterion for distributivity.
\begin{corollary}
  A finite lattice is distributive if and only if for all covering pairs the meet-rises and join-rises are unit.
\end{corollary}

It is also well-known (see Theorem~\ref{theorem-M3N5}) that a lattice is distributive if and only if it does not contain $M_3$ or $N_5$ (as given in Figures~\ref{fig-m3} and \ref{fig-n5}) as sublattice. If it contains $M_3$, then it
is neither join- nor meet-distributive and hence must have both non-unit
join-rises and non-unit meet-rises (see \Cref{fig:m3-m3-rise} for an example). If
it contains $N_5$, then it has a non-unit join-rise or a non-unit meet-rise. %
It is not necessary that it has both of them, as shown for instance for the
lattice $S_7^d$ in \Cref{fig:s7-dual-rise}. We now turn to the natural question of the properties that arise from unit rises
in a single direction. The following theorem offers a characterization in terms
of a sublattice condition.

\begin{figure}[t]
  \centering \includegraphics{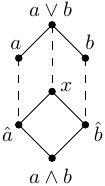}
  \caption{%
    Visualization of the proof of \Cref{theorem:s7-sublattice}. Solid lines indicate a covering relation, while dashed lines indicate comparability. %
  }
  \label{fig:theorem-5}
\end{figure}
\begin{theorem}
  \label{theorem:s7-sublattice}
  Let $\uL$ be a finite, join-distributive lattice. Then $\uL$ is
  meet-distributive if and only if it does not contain $S_7$ as a sublattice.
\end{theorem}
\begin{proof}``$\Rightarrow$'': If $\uL$ is meet-distributive  then it cannot have a non-meet-distributive sublattice. In particular it cannot have $S_7$ as sublattice, because $S_7$ is not meetdistributive.

  ``$\Leftarrow$'': Let $\uL$ be a finite join-distributive lattice that is not
  meet-distributive. We will construct a sublattice of $\uL$ that is isomorphic
  to $S_7$.
  First, we show that $\uL$ is %
  not dually semimodular: Assume that it were dually semimodular. As dual
  semimodularity together with join-distributivity implies distributivity, and
  thus also meet-distributivity, $\uL$ would also be meet-distributive.
  \Lightning.

  As $\uL$ is not dually semimodular, it does in particular not fulfill the
  weak condition of dual semimodularity since $L$ is finite \cite{roman2008lattices}, which would prohibit the existence of two incomparable
  elements $a,b \in L$ with $a\coveredby a \up b$, $b \coveredby a \up b$ and
  $a\down b \ncoveredby a$. Let $a$ and $b$ be two such elements.
  As $\uL$ is join-distributive it is, by Item 5 of Theorem~\ref{theorem-joindistributivity}, graded. Hence also the interval $[a\wedge b, a\vee b]$ is graded. Because of $a\down b \ncoveredby a$ there exists $\hat{a}$
  with $a\wedge b \coveredby \hat{a} < a$, and because of the gradedness of the
  interval there also exists $\hat{b}$ with $a\wedge b \coveredby \hat{b} < b$.
  Let $x:=\hat{a} \up \hat{b}$. In the remainder of the proof we will show that
  the seven elements $ a \up b, a, b, x, \hat{a}, \hat{b}$ and $a \down b$
  constitute a sublattice of $\uL$ which is isomorphic to $S_7$ (see
  \Cref{fig:theorem-5}):
  \begin{itemize}
    \item Note that $\hat{b} \nleq a$ and $\hat{a} \nleq b$, otherwise $\hat{a} =
            a \down b$ or $\hat{b} = a \down b$.
          Thus, $a\wedge b\leq  \hat{a}\wedge \hat{b} < \hat{a}$, and hence $a\wedge b = \hat{a}\wedge\hat{b}$.
    \item As $\uL$ is join-distributive it is also semimodular. Hence $a\wedge b \coveredby \hat{a}$ and
          $a\wedge b \coveredby \hat{b}$ imply $ \hat{a}\coveredby x$ and
          $\hat{b}\coveredby x$, resp.
    \item We have $x \nleq a$ as otherwise we would have $\hat{b}\leq x \leq a$.
          The same type of argument yields $x \nleq b$.
    \item Furthermore, $x = \hat{a} \up \hat{b} \leq a \up b$. Also, $x \neq a \up
            b$ as otherwise there would be a maximal chain with three elements from $a
            \up b$ to $a \down b$, contradicting the gradedness of the interval
          $[a\wedge b, a\vee b]$.
    \item As $x \leq a \up b$, we know that $a \up x \leq a \up b$. %
          Furthermore, $a \up x = a \up b$, as otherwise $a < a \up x < a \up b$ which
          would contradict $a \coveredby a \up b$. %
          By the same argument $b \up x = a \up b$.
    \item Additionally, $a \down x = \hat{a}$, as otherwise $x > a \down x >
            \hat{a}$, a contradiction to $x \covers \hat{a}$. %
          Analogously, $b \down x = \hat{b}$.      \proofbox
  \end{itemize}
\end{proof}

\section{Join-distributivity in Posets} %

Earlier in this work, we examined join-distributivity in lattices by measuring the number of non-unit rises.
In this section, we transfer this approach to ordered sets.
In data analysis contexts, ordered sets can be considered as condensed representation of the full lattice (via the Dedekind-MacNeille completion). To this end, we first need the notion of join-distributivity in ordered sets.

\begin{definition}
    An ordered set $(P,{\leq})$ is called \emph{join-distributive}, if $DM(P,{\leq})$ is join-distributive. Dually, it is called \emph{meet-distributive} if $DM(P,{\leq})$ is meet-distributive. Finally, it is called \emph{distributive} if it is meet- and join-distributive.
\end{definition}
This definition follows the style of definition of J. Larmerová and J. Rachůnek \cite{larmerova1988translations} via the Dedekind-MacNeille completion, and our definition of a distributive ordered set coincides with theirs. Other properties of lattices such as being modular, complemented, Boolean, pseudo-complemented, Brouwerian and Stone have been adapted to ordered sets \cite{larmerova1988translations,chajda1989forbidden,halas1993pseudocomplemented}, but up to our knowledge join-distributivity is not among them.

Recall, that join- and meet-irreducible elements in ordered sets, are defined the same way as in lattices and that $\mathcal{J}(\uP)$ and $\mathcal{M}(\uP)$ refers to the set of join-irreducible and meet-irreducible elements respectively.
Note that if an ordered set has a smallest element $\bot$, then it is join-reducible as the supremum of the empty set.
On the other hand, in an ordered set with multiple minimal elements, each of them is irreducible.
With keeping this in mind, \emph{join-rises} and \emph{meet-rises in ordered sets} are defined the same way as in lattices.

\cref{fig:dmc-step-numbers} shows two ordered sets and their Dedekind-MacNeille
completions.
The numbers refer to the number of meet-irreducible elements which are larger or equal to each element.
It is apparent that in the examples the numbers are preserved, when the Dedekind-MacNeille
completion is computed, which implies that the rises between elements are
preserved.
The following proposition shows that this is indeed the case. The first part of the proposition is only given because it is used to prove the second part. It seems to be folklore in lattice theory, but we couldn't find it stated explicitly in the literature.
\begin{proposition}
    \label{prop-j=j}
    Let $\uP$ be a finite ordered set.
    \begin{enumerate}
        \item $\mathcal{M}(DM(\uP)) = \{\iota (x) \mid x\in  \mathcal{M}(\uP)\}$
        \item For $x\in P$ the equality $m_{\uP}(x)=m_{DM(\uP)}(\iota (x))$ holds. %
    \end{enumerate}
\end{proposition}
\begin{proof}

    1. `$\subseteq$':
    All elements in $DM(\uP)\setminus \{\iota(x) \mid x\in P\}$ are reducible by the construction of the DM completion, as  they are created as joins and meets of the original elements. Hence $\mathcal{M}(DM(\uP)) \subseteq \{\iota (x) \mid x\in  \mathcal{M}(\uP)\}$.

    1. `$\supseteq$':
    Assume the opposite, i.e., there exists an $x \in \mathcal{M}(P)$ with $\iota(x)\notin \mathcal{M}(DM(\uP))$.
    Then there also exists $\hat{K} \subset \mathcal{M}(DM(\uP))$ with $\iota(x)=\bigwedge \hat{K}$.
    By `$\subseteq$', for each $\hat{k} \in \hat{K}$ there is a $k \in \mathcal{M}(\uP)$ with $\iota(k)=\hat{k}$, and, as $\iota$ is an embedding and therefore injective, $k$ is even unique in $P$.
    Let $K \subset \mathcal{M(\uP)}$ such that $\iota(K) = \hat{K}$.
    As $\iota$ is order-preserving, $x \leq k$ for all $k \in K$.
    Thus, there is some element $y$ that is a maximal lower bound of $K$.
    As $\iota$ is order preserving and $\iota(K)=\hat{K}$, it follows that $\iota(y)\leq \bigwedge\hat{K}=\iota(x)$ and therefore $y \leq x$.
    As $\iota(x)=\bigwedge \hat{K}\leq \iota(k)$ for any $k \in K$, $x$  is a common, lower bound of $K$.
    This makes $x$ the unique greatest lower bound of $K$, i.e., $x = \bigwedge K$, which contradicts the original assumption.

    Ad 2: With Item 1 and the fact that $\iota$ is an order embedding, we obtain $m_{\uP}(x) = |\{y\in\mathcal{M}(\uP) \mid y\geq_{\uP} x\}| = |\{y\in\mathcal{M}(DM(\uP)) \mid y\geq_{DM(\uP)} \iota (x)\}| = m_{DM(\uP)}(\iota (x))$.
    \proofbox

\end{proof}

\begin{figure}[t]
    \centering
    \null\hfill
    \includegraphics{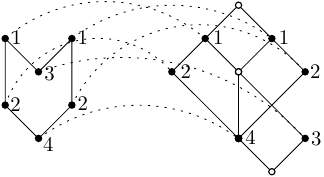}
    \hfill
    \raisebox{4.5em}{
        \begin{cxt}
            \att{$m_1$}
            \att{$m_2$}
            \att{$m_3$}
            \att{$m_4$}
            \att{$m_5$}
            \obj{xbxd.}{$g_1$}
            \obj{bxdx.}{$g_2$}
            \obj{xxbbx}{$g_3$}
            \obj{xxxxb}{$g_4$}
        \end{cxt}}
    \hfill\null\\
    \null\hfill
    \includegraphics{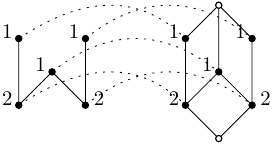}
    \hfill
    \raisebox{4.5em}{
        \begin{cxt}
            \att{$m_1$}
            \att{$m_2$}
            \att{$m_3$}
            \obj{xbu}{$g_1$}
            \obj{ubx}{$g_2$}
            \obj{xxb}{$g_3$}
            \obj{bxx}{$g_4$}
        \end{cxt}}
    \hfill\null
    \caption{%
        Two examples for the
        Dedekind-MacNeille completion, showing that the step numbers stay the same (element wise). %
        Numbers indicate $m(x)$, the white nodes are the ones that are added in the Dedekind-MacNeille completion.
    }
    \label{fig:dmc-step-numbers}
\end{figure}

In the sequel we want to analyze, how unit rises in the ordered set are preserved under the Dedekind-MacNeille completion.

\begin{proposition}
    Let $\uP=(P,\leq)$ be a finite ordered set. If $x\prec y$ in $\uP$ and $\Delta_m(x,y)=1$ then $\iota(x)\prec\iota(y)$ in $DM(\uP)$.

\end{proposition}

\begin{proof}
    Let $x \prec y$  in $(P,\leq)$ with a unit meet-rise. With Proposition~\ref{prop-j=j}(2) follows $\Delta_m(\iota(x),\iota(y))=\Delta_m(x,y)=1$. Hence, there cannot exist some $z\in DM(\uP)$ with  $\iota(x) < z < \iota(y)$.

\end{proof}

To determine if a finite ordered set is join-distributive, one could compute the lattice completion of the ordered set and then check, if there are no non-unit meet-rises.
However, this is potentially expensive, as the number of elements of the Dedekind-MacNeille completion might be exponentially larger than the number of elements in the original ordered set.
We thus propose another way, which can be determined in polynomial time with respect to the number of elements of the ordered set, using the following isomorphism that follows directly from the standard context of a lattice as described in~\cite[Proposition 14]{GanterFormal2024}.

\begin{lemma}
    Let $(P,\leq)$ be a finite ordered set. Then $\uB(J,M,{\leq}\vert_{J \cup M})$ with $J:=\mathcal{J}(\uP)$ and  $M:=\mathcal{M}(\uP)$ is isomorphic to $DM(P,\leq)$.
\end{lemma}

We can thus compute the context $(J,M,{\leq}\vert_{J \cup M})$ and then use the equivalence of (1)\ and (9)\ in \cref{theorem-joindistributivity} to determine join-distributivity.
As this context is clarified by definition, we only have to check if every row that has a doublearrow has no additional downarrow or doublearrow.
Computing the arrow relations can, by their definition, be done in polynomial time (\eg, by checking their definition for each cell), as the necessary inclusion checks between attribute and object vectors are at most quadratic in number, respectively.
We thus receive a polynomial-time check for join-distributivity of arbitrary finite ordered sets.
As example, the right side of \cref{fig:dmc-step-numbers} shows the formal contexts $(J,M,{\leq}\vert_{J \cup M})$ and their arrow relations.
For the upper ordered set, $g_1, g_2$ and $g_3$ have at least one doublearrow and an additional downarrow or doublearrow, thus, the ordered set is not join-distributive.
On the other hand there is no attribute with a doublearrow and an additional uparrow or doublearrow which makes the ordered set meet-distributive.
The opposite is true for the lower ordered set, which is not meet- but join-distributive.

\section{Local distributivity in real-world data}
\label{sec-realworld}

Join- and meet-rises provide us with a straightforward way of measuring
non-distributivity in real-world (concept) lattices:%
\footnote{There are many other, equally interesting ways for measuring
  non-distributivity, which we do not discuss here. Examples include
  $|\{(x,y,z)\in\uL^3 \mid x\vee(y\wedge z) \neq (x\vee y)\wedge (x\vee z\}|$
  (and variations thereof which only count the non-trivial tuples) and
  $|\{a\in\uL \mid \exists x,y\in\uL\colon a\wedge x = a\wedge y, a\vee y =
    a\vee y, x\neq y\}|$.}
\begin{definition}
  Let $\uL$ be a finite lattice. Then $nur_\wedge(\uL):=|\{(x,y)\in\uL^2 \mid
    x\prec y, \Delta_m (x,y) > 1 \}|$ is the \emph{absolute number of
    non-unit meet-rises}. When normalising this value by dividing it by
  $\mid\prec\mid$ (the number of covering pairs), one obtains the \emph{relative
    number of non-unit meet-rises}. We will abuse the notation by using
  $nur_\wedge(\uL)$ also for the latter, whenever this is clear from the
  context. $nur_\vee(\uL)$, the \emph{(absolute and relative) number of non-unit
    join-rises} is defined dually.
\end{definition}
The following Lemma follows directly from Theorem~\ref{theorem-our-joindistributivity}. It summarizes the relationship between the $nur_\wedge(\uL)$ and the $nur_\vee(\uL)$ and the different forms of distributivity.

\begin{lemma}
  A finite lattice $\uL$ is
  \begin{itemize}
    \item join-distributive iff $nur_\wedge(\uL) = 0\,\%$,
    \item meet-distributive iff $nur_\vee(\uL) = 0\,\%$, and
    \item distributive iff $nur_\wedge(\uL) = nur_\vee(\uL) = 0\,\%$.
  \end{itemize}
\end{lemma}
The lattice is thus `100\,\% distributive' if the relative frequency of the meet- and join-rises that are non-unit is 0\,\%.
The more these values deviate from 0\,\%, the `less distributive' is the lattice.

In \cite{WilleTruncated2003}, R. Wille observed that many real-world concept lattices are (mostly) distributive at the top because ``everyday thinking predominantly uses logical inferences with one-element premises'', and then abruptly break down at the bottom because of conflicting attributes. While this led Wille to introduce truncated distributive lattices, it inspired us to the following assumption:

\medskip\noindent
\textbf{Assumption A1.}
Concept lattices of real-world contexts will show a higher degree of join-distributivity than of meet-distributivity.
\medskip

If this assumption holds we would expect that $nur_\wedge(\uL)\leq
  nur_\vee(\uL)$ in most cases. To this end, we have analysed the non-unit rises
of all concept lattices of the FCA context
repository\footnote{\url{https://github.com/fcatools/contexts/tree/main/contexts}
  (visited March 25, 2025)} \cite{hanika2024repository} (top part of
Table~\ref{table-realworld}) and of some further datasets (bottom part of the
table). The datasets were not preprocessed and all analysis was conducted using
custom Python code. In the table, $\uL$ is the concept lattice $\uB(\K)$ of the
context at hand, and $\mid\prec\mid$ gives the number of covering pairs in $\uL$.

\begin{table}[htbp]
	\centering
	\caption{$nur$-values for some real-world datasets}
	\label{table-realworld}
	\begin{tabular}{x{.75pt}lx{.75pt}rx{.05pt}rx{.75pt}rx{.1pt}rx{.75pt}rx{.1pt}rx{.75pt}}
		\toprule
		\multirowcell{3}{$\K$}
		                                                                    & \multirowcell{3}{$|\uL|$}
		                                                                    & \multirowcell{3}{$|\prec|$} &
		\multicolumn{2}{cx{.75pt}}{$nur_\vee(\uL)$}                         &

		\multicolumn{2}{cx{.75pt}}{$nur_\wedge(\uL)$}
		\\
		                                                                    &                             &               & abs.          & rel.   & abs.       & rel.   \\
		\midrule
		officesupplies                                                      & 5                           & 5             & 1             & 0.20   & 1          & 0.20   \\
		newzealand                                                          & 8                           & 10            & 0             & 0.00   & 0          & 0.00   \\
		planets                                                             & 12                          & 18            & 5             & 0.28   & 5          & 0.28   \\
		bodiesofwater                                                       & 12                          & 18            & 6             & 0.33   & 1          & 0.06   \\
		famous\_animals                                                     & 13                          & 21            & 6             & 0.29   & 5          & 0.24   \\
		missmarple                                                          & 13                          & 21            & 12            & 0.57   & 5          & 0.24   \\
		livingbeings                                                        & 19                          & 32            & 11            & 0.34   & 8          & 0.25   \\
		driveconcepts                                                       & 24                          & 50            & 2             & 0.04   & 12         & 0.24   \\
		gewässer                                                            & 28                          & 62            & 30            & 0.48   & 8          & 0.13   \\
		animals                                                             & 35                          & 66            & 33            & 0.50   & 20         & 0.30   \\
		tealady                                                             & 65                          & 148           & 48            & 0.32   & 31         & 0.21   \\
		music                                                               & 163                         & 507           & 311           & 0.61   & 56         & 0.11   \\
		seasoningplanner \quad                                              & 532                         & 1593          & 784           & 0.49   & 563        & 0.35   \\
		\hline
		zoo \cite{zoo_111}                                                  & 4\,579                      & 19\,076       & 12\,228       & 0.64   & 3\,451       & 0.18   \\
		wikipedia \cite{kunegis2013konect}                                  & 14\,171                     & 70\,732       & 53\,342       & 0.75   & 6\,161       & 0.09   \\
		students \cite{students}                                            & 17\,603                     & 80\,902       & 67\,439       & 0.83   & 5\,713       & 0.07   \\
		wiki44k \cite{DBLP:conf/icfca/Hanika0S19,DBLP:conf/semweb/HoSGKW18} & 21\,923                     & 109\,698      & 91\,924       & 0.84   & 4\,526       & 0.04   \\
		mushroom \cite{mushroom_73}                                         & \ 238\,710                  & \ 1\,370\,991 & \ 1\,318\,000 & \ 0.96 & \ 102\,519 & \ 0.07 \\
		\bottomrule
	\end{tabular}
\end{table}

First of all, we observe that the only distributive concept lattice in this collection is the one about New Zealand (which is isomorphic to $\{0,1,2,3\}\times\{0,1\}$), as it is the only one with $nur_\vee(\uL)=nur_\wedge(\uL)=0$. Then we observe that for all datasets (except the drive concepts), we have indeed $nur_\wedge(\uL)\leq nur_\vee(\uL)$, supporting our assumption.
We assume that the concept lattice on drive concepts behaves differently because it has an atypcial small number of objects (5) for a comparatively large number (25) of attributes.

The larger the lattice, the more pronounced is the effect postulated in A1, as can be seen in Figure~\ref{fig:scatter}. With increasing size of the covering relation, the relative number of non-unit meet-rises is (with very few exceptions) decreasing, while the relative number of non-unit join-rises is with almost no exceptions increasing.
The extreme case is the mushroom dataset, whose concept lattice is `96\,\% non-meet-distributive' but only `7\,\% non-join-distributive'.
\begin{figure}[htbp]
  \centering
  \includegraphics[width=0.7\textwidth]{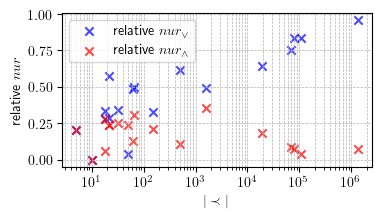}
  \caption{The plot shows for every lattice its relative $nur_\vee$ (red) and $nur_\wedge$ (blue) values in relation to the size of its covering relation.
  }
  \label{fig:scatter}
\end{figure}

Coming back to Wille's observation that distributivity breaks down at the bottom of a concept lattice because of inconsistent attribute combinations, we assume that the bottom element of the lattice is particularly prone to generate non-unit meet-rises.
Therefore, we analyse the count of non-unit meet-rises on the lowest covering pairs (\ie, those between the bottom element of the lattice and the atoms) in the first row of Table~\ref{table-atoms}. This  shows that the concept lattices of real-world datasets exhibit almost exclusively non-unit meet-rises at the very bottom. This comes with no surprise, as a unit-rise would only occur if the object that is generating the atom concept at hand has (up to purification) exactly all but one attribute of the formal context which is rarely the case.
\begin{table}[htbp]
	\centering
	\caption{Relative non-unit rises on the atoms and co-atoms of the concept lattice}
	\label{table-atoms}
	\begin{tabular}{lcccccccccccccccccc}
		\toprule
		                        & \rotatebox{90}{officesupplies} & \rotatebox{90}{newzealand} & \rotatebox{90}{planets}    & \rotatebox{90}{bodiesofwater} & \rotatebox{90}{famous\_animals} & \rotatebox{90}{missmarple} & \rotatebox{90}{livingbeings} & \rotatebox{90}{driveconcepts} & \rotatebox{90}{gewässer}   & \rotatebox{90}{animals}    & \rotatebox{90}{tealady}    & \rotatebox{90}{music}      & \rule{0pt}{2.6ex} \rotatebox{90}{seasoningplanner\hspace*{1ex}} \rule[-1.2ex]{0pt}{0pt}& \rotatebox{90}{zoo}          & \rotatebox{90}{wikipedia}      & \rotatebox{90}{students}       & \rotatebox{90}{wiki44k}        & \rotatebox{90}{mushroom}         \\
		\midrule
\rule{0pt}{4ex}		$nur_\wedge(\uL)/$atoms & $\displaystyle\frac{1}{2}$     & $\displaystyle\frac{0}{2}$ & $\displaystyle\frac{5}{5}$ & $\displaystyle\frac{1}{2}$    & $\displaystyle\frac{5}{5}$      & $\displaystyle\frac{5}{5}$ & $\displaystyle\frac{4}{4}$   & $\displaystyle\frac{4}{5}$    & $\displaystyle\frac{8}{8}$ & $\displaystyle\frac{8}{8}$ & $\displaystyle\frac{7}{7}$ & $\displaystyle\frac{5}{6}$ & $\displaystyle\frac{38}{38}$     & $\displaystyle\frac{59}{59}$ & $\displaystyle\frac{103}{103}$ & $\displaystyle\frac{479}{479}$ & $\displaystyle\frac{587}{587}$ & $\displaystyle\frac{8124}{8124}$
\rule[-3ex]{0pt}{0pt}
		\\ \midrule
\rule{0pt}{4ex}		$nur_\vee(\uL)/$coatoms\hspace*{1cm} & $\displaystyle\frac{1}{2}$     & $\displaystyle\frac{0}{2}$ & $\displaystyle\frac{1}{2}$ & $\displaystyle\frac{2}{3}$    & $\displaystyle\frac{1}{2}$      & $\displaystyle\frac{4}{4}$ & $\displaystyle\frac{4}{4}$   & $\displaystyle\frac{1}{4}$    & $\displaystyle\frac{6}{6}$ & $\displaystyle\frac{5}{5}$ & $\displaystyle\frac{7}{7}$ & $\displaystyle\frac{7}{7}$ & $\displaystyle\frac{29}{29}$     & $\displaystyle\frac{17}{17}$ & $\displaystyle\frac{68}{68}$   & $\displaystyle\frac{27}{27}$   & $\displaystyle\frac{65}{65}$   & $\displaystyle\frac{23}{23}$    
\rule[-3ex]{0pt}{0pt}
		 \\				
		\bottomrule
	\end{tabular}
\end{table}

The same observation holds for the join-rises between the top-element and the
co-atoms of the lattice, as shown in the second row of Table~\ref{table-atoms}.
For the ten largest of our datasets (from \emph{gewässer} to \emph{mushroom}),
all join-rises between co-atoms and top element are non-unit. This also comes
with no real surprise, as a unit join-rise would mean that the attribute that is
generating the co-atom under study would match with all objects beside exactly
one. However, this effect does not contribute as much to the total count of
non-unit join-rises compared to the effect of the atoms to the total count of
non-unit meet-rises, as real-world concept lattices tend to be leaner at the
top: the wiki44k dataset, for instance, has 587 atoms but only 65 co-atoms, and
the mushroom dataset has 8\,124 atoms but only 23 co-atoms.

Further insight into the distribution of the non-unit meet-rises is provided by
Figure~\ref{fig:heightvsnur}. There we have plotted for all lattices with height
10 or more the distribution of $nur_\wedge$ over the height of the covering
pairs in the lattice. (The plots for the smaller lattices look very similar.)
The left-most point of each curve corresponds to the 100\,\% entries in
Table~\ref{table-atoms}. We observe that the values are then almost everywhere
decreasing, which means that the non-join-distributivity in these lattices
accumulates towards the bottom of the concept lattice. An extreme case is the
mushroom dataset, which is `join-distributive almost everywhere' with the
exception of the atoms of the lattice.
\begin{figure}[htbp]
  \subfloat[mushroom]{\includegraphics[width=0.5\textwidth]{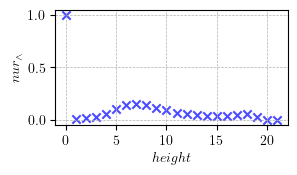} }
  \subfloat[seasoningplanner]{\includegraphics[width=0.5\textwidth]{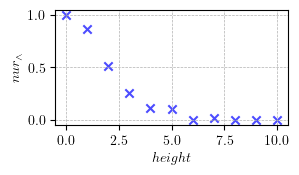} }\\[1ex]
  \subfloat[students]{\includegraphics[width=0.5\textwidth]{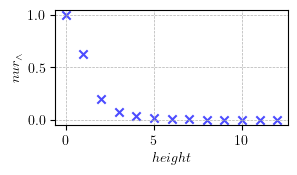} }
  \subfloat[wiki44k]{\includegraphics[width=0.5\textwidth]{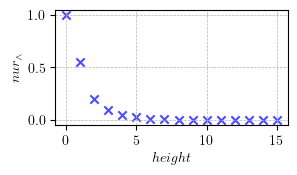} }\\[1ex]
  \subfloat[wikipedia]{\includegraphics[width=0.5\textwidth]{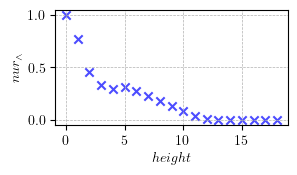} }
  \subfloat[zoo]{\includegraphics[width=0.5\textwidth]{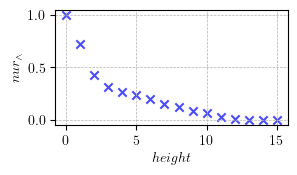} }
  \caption{Distribution of $nur_\wedge$ over the height for all lattices with height $\geq 10$.}
  \label{fig:heightvsnur}

\end{figure}

\section{Conclusion and Future Work}
\label{sec-conclusion}

In this paper, we have introduced non-unit rises as indicator and measures for the degree of the presence of local distributivity in (concept) lattices and ordered sets. This opens various lines of research that may at the end lead to algebraic methods for analysing large real-world concept lattices: How can the failure of local distributivity be repaired? How can the result be used for novel drawing algorithms for concept lattices? How can it be exploited for complexity reduction by means of congruence and tolerance relations? Last but not least it is an interesting open question how (join)distributive ordered sets can be defined directly, without the detour via the Dedekind-MacNeille completion.

\paragraph{Acknowledgment.} We thank the anonymous reviewers for their valuable comments.

\bibliographystyle{splncs04}
\bibliography{paper}

\end{document}